\documentclass[letterpaper, 10 pt, conference]{ieeeconf}  

\IEEEoverridecommandlockouts                              

\usepackage{graphics}
\usepackage{amsmath}

\usepackage{amsthm}
\usepackage{amssymb}
\usepackage{mathtools}
\usepackage{bbm}
\usepackage{dsfont}
\usepackage{color}
\usepackage{xcolor}
\definecolor{myblue}{rgb}{0 0.4470 0.7410}
\definecolor{mypurple}{rgb}{0.4940 0.1840 0.5560}
\definecolor{mygreen}{rgb}{0.4660 0.6740 0.1880}
\usepackage[noend]{algpseudocode}
\usepackage[ruled, vlined, linesnumbered]{algorithm2e}
\usepackage{algorithmicx}
\newtheorem{lem}{Lemma}

\newtheorem{prob}{Problem}
\newtheorem{theorem}{Theorem}

\newtheorem{ex}{Example}

\newtheorem{remark}{Remark}
\theoremstyle{remark}

\usepackage{pgfplots}
\pgfplotsset{compat=newest}
\pgfplotsset{plot coordinates/math parser=false}
\newlength\figureheight
\newlength\figurewidth

\newcommand\norm[1]{\left\lVert#1\right\rVert}

\title{\LARGE \bf
Task Hierarchical Control via Null-Space \\Projection and Path Integral Approach 
}

\author{Apurva Patil$^{1}$, Riku Funada$^{2}$, Takashi Tanaka$^{3}$, Luis Sentis$^{3}$\thanks{This work was supported in part by Office of Naval Research (N00014-22-1-2204) and AFOSR Grant FA9550-20-1-0101. $^{1}$Walker Department of Mechanical Engineering, University of Texas at Austin, {\tt\small apurvapatil@utexas.edu}.
$^{2}$Department of Systems and Control, Institute of Science Tokyo, {\tt\small funada@sc.e.titech.ac.jp}.
$^{3}$Department of Aerospace Engineering and Engineering Mechanics, University of Texas at Austin, {\tt\small ttanaka@utexas.edu}, {\tt\small lsentis@utexas.edu
}.}
}

\begin{document}

\maketitle
\thispagestyle{empty}
\pagestyle{empty}

\begin{abstract}

This paper addresses the problem of hierarchical task control, where a robotic system must perform multiple subtasks with varying levels of priority. A commonly used approach for hierarchical control is the null-space projection technique, which ensures that higher-priority tasks are executed without interference from lower-priority ones. While effective, the state-of-the-art implementations of this method rely on low-level controllers, such as PID controllers, which can be prone to suboptimal solutions in complex tasks. This paper presents a novel framework for hierarchical task control, integrating the null-space projection technique with the path integral control method. Our approach leverages Monte Carlo simulations for real-time computation of optimal control inputs, allowing for the seamless integration of simpler PID-like controllers with a more sophisticated optimal control technique.  Through simulation studies, we demonstrate the effectiveness of this combined approach, showing how it overcomes the limitations of traditional methods by optimizing the task performance.

\end{abstract}

\section{Introduction}\label{sec: Introduction}
Robotic systems with a large number of degrees of freedom offer significant versatility; however, this also introduces redundancies. These systems are often used to accomplish multiple subtasks with varying levels of importance, allowing for the establishment of a task hierarchy. One of the most frequently applied methods to accomplish task hierarchical control is the null-space projection technique \cite{antonelli2008null, khatib1987unified, slotine1991general}. In the null-space projection, the top priority task is executed by employing all the capabilities of the system. The second priority task is then applied to the null space of the top priority task. In other words, the task on the second level is executed as well as possible without disturbing or interfering with the first level. The task on level three is then executed without disturbing the two higher-priority tasks, and so forth \cite{dietrich2015overview}. The null-space projection technique is based on a hierarchical arrangement of the involved tasks and can be interpreted as instantaneous local optimization. \par

The null-space projection technique has been widely applied, such as in multi-robot team control \cite{sentis2009large}, whole-body behavior synthesis \cite{sentis2005synthesis}, and manipulator control \cite{dietrich2015overview}. A comprehensive overview and comparison of null-space projections is given in \cite{dietrich2015overview}. In the literature \cite{antonelli2008null, dietrich2015overview}, the individual controllers for the tasks in the hierarchy are designed using simple low-level controllers such as proportional-integral-derivative (PID) controllers. While these controllers are easy to design and are capable of generating control inputs in real time, there is no systematic way to optimize the overall performance of the hierarchy of controllers.  
Global optimization techniques such as dynamic programming on the other hand minimize some performance index across a whole trajectory. For example, to design optimal control policies for a linear system with a quadratic cost function, the linear quadratic regulator (LQR) is widely used. For nonlinear systems and cost functions, approaches such as iterative LQR (ILQR), and differential dynamic programming (DDP) can be utilized. However, even though global optimization solutions perform better than local optimization solutions, they are impractical for online feedback control, due to the heavy computational requirements. In this paper, we employ the path integral control method, a stochastic optimal control framework that can be applied to nonlinear systems and enables the computation of optimal control inputs in real time through Monte Carlo simulations. \par

Based on the foundational work of \cite{kappen2005path, theodorou2010generalized}, the path integral method can be defined as a sampling-based algorithm to compute control input at each time step from a large number of Monte-Carlo simulations of the \emph{uncontrolled dynamics}. Unlike traditional optimal control methods, the path integral approach can directly deal with stochasticity and nonlinearity \cite{vrushabh2020robust, patil2022chance}, \cite{patil2023simulator, patil2023risk}. Moreover, unlike dynamic programming, it can evaluate control input without solving a high-dimensional Hamilton-Jacobi-Bellman partial differential equation \cite{patil2022chance}. The Monte Carlo simulations can also be highly parallelized, leveraging the GPU resources available on modern robotics platforms \cite{williams2017model}, making it particularly effective for real-time control applications.\par


This paper integrates the path integral control approach with the null-space projection technique to overcome their individual limitations and enhance their respective strengths. We explain our idea via the following example:

\begin{ex}

Consider a platoon of robots navigating in an obstacle-filled environment, tasked with three goals in descending order of importance:
\begin{enumerate}
    \item Avoiding collisions with obstacles
    \item Steering the platoon’s centroid toward a goal position
    \item Maintaining specific distances between the robots
\end{enumerate}
Designing an optimal controller using only the path integral method for each robot in the platoon would present scalability challenges due to the method's sampling-based nature. On the other hand, simple low-level controllers (such as PID), while computationally efficient, are difficult to tune manually for a better performance. We propose using local controllers for tasks 1 and 3 while applying the path integral controller to the more complex task 2. In this way, the path integral controller optimizes the centroid’s movement, while simpler tasks are handled by local controllers.
\end{ex} 

In \cite{sentis2009large}, the authors address the issue of the task hierarchical controller falling into local minima by complementing the low-level controllers with the A* search algorithm. In graph-based methods like A*, the optimal trajectory is generated in the workspace, requiring an additional tracking controller to be designed separately, taking into account the system's dynamics. In contrast, the path integral controller computes the optimal policy directly, eliminating the need for such a decoupled approach. Furthermore, path integral controllers can adapt in real time making them particularly suitable for dynamic environments. Graph-based algorithms like A* are primarily designed for static environments and are less effective in dynamic scenarios, as they require re-planning the trajectory each time the environment changes. Although there are graph-based algorithms, such as RRTX, capable of handling dynamic environments \cite{otte2016rrtx}, these approaches tend to have higher computational costs compared to static planners and often require substantial memory resources as the number of samples increases.\par

The contributions of this paper are as follows:
 (1) We introduce a new framework for hierarchical task control that combines the null-space projection technique with the path integral control method. This leverages Monte Carlo simulations for real-time computation of optimal control inputs, allowing for the seamless integration of simpler PID-like controllers with a more sophisticated optimal control technique. (2) Despite the wide applicability of the path integral approach, it has not been utilized for solving the task hierarchical control problem to the best of the authors' knowledge. This expands the applicability of path integral control to multi-task robotic systems, enabling a more robust handling of task prioritization. (3) Our simulation studies demonstrate the effectiveness of the proposed approach, showing how it overcomes the limitations of the state-of-the-art methods by optimizing task performance.
  
\section{Preliminaries}\label{Sec: preliminaries}
 Let $q\in\mathbb{R}^n$ be the configuration of a robot, where $n$ is the number of degrees-of-freedom (DOFs). We introduce $K$ task variables 
\begin{equation}\label{task function}
  \sigma_k = h_k(q)\in\mathbb{R}^{m_k}, \quad k\in \mathcal{K} 
\end{equation}
 where $m_k$ is the dimension of task $k$ and $\mathcal{K}=\{1, 2, ... , K\}$ denotes a set of the indices. The hierarchy is defined such that $\sigma_{1}$ is at the top priority and $\sigma_i$ is located higher in the priority order than $\sigma_j$ if $i<j $. The task variables $\sigma_k$ represent functional quantities (e.g., a cost or a potential function) as part of the desired actions, and $h_k$ is a differentiable nonlinear function. Our goal is to devise a policy for the robotic system that would accomplish the $K$ subtasks $\sigma_k, k\in \mathcal{K}$ in their descending order of importance. A task $k$ gets accomplished if $\sigma_k$ converges to the desired task trajectory $\sigma_{k,d}$. In the rest of the paper, for notational compactness, the functional dependency on $q$ is dropped whenever it is unambiguous. \par

Differentiating \eqref{task function} with respect to time, we get 
\begin{equation}\label{Jacobian}
    \dot{\sigma}_k = J_k(q)\dot{q}, \qquad J_k(q) = \frac {\partial h_k(q)}{\partial q}
\end{equation}
where $J_k(q)\in\mathbb{R}^{m_k\times n}$ is the Jacobian matrix and $\dot{q}$ represents the velocity of the robot in the configuration space. This velocity can be computed by inverting the mapping \eqref{Jacobian} \cite{siciliano1990kinematic}, \cite{antonelli2009stability}, \cite{antonelli2009prioritized}. However, in the case of a redundant system i.e., when $n>m_k$, the problem \eqref{Jacobian} admits infinite solutions. A common approach is to solve for the minimum-norm velocity, which leads to the least-squares solution $\dot{q} = J^\dag_k(q)\dot{\sigma}_k$ where $J_k^\dag(q) = J_k^\top(q)(J_k(q)J_k^\top(q))^{-1}$ is the right pseudo-inverse of the Jacobian matrix $J_k(q)$. In the following $J_k(q)$ is assumed to be non-singular, hence of full row rank. Similar to \eqref{Jacobian}, the acceleration in the configuration space can be computed by further differentiating \eqref{Jacobian}: $ \ddot{\sigma}_k = J_k(q)\ddot{q} + \dot{J}_k(q)\dot{q}.$ The minimum-norm solution for the acceleration $\ddot{q}$ is obtained as:
\begin{equation}\label{q ddot}
   \ddot{q} = J^\dag_k(q) \left(\ddot{\sigma}_k - \dot{J}_k(q)\dot{q} \right). 
\end{equation}
Equation \eqref{q ddot} provides a basic method to compute system acceleration in an open-loop style. To improve convergence, a feedback term is added to \eqref{q ddot}, as suggested by \cite{siciliano1990kinematic} and \cite{tsai1987strictly}, leading to the expanded form:
\begin{equation}\label{q ddot PD}
   \ddot{q} = J^\dag_k(q) \left(\left\{\ddot{\sigma}_{k,d} + K_{p,k}\widetilde{\sigma}_k + K_{d,k}\frac{d\widetilde{\sigma}_k}{dt}\right\} -\dot{J}_k(q)\dot{q} \right) 
\end{equation}
where $\widetilde{\sigma}_k = \sigma_{k,d} - \sigma_{k}$ denotes the error between the desired task trajectory $\sigma_{k,d}$ and the actual task trajectory $\sigma_{k}$ which can be computed from the system's current configurations using \eqref{task function}. For task $k$, the terms $K_{p,k}$ and $K_{d,k}$ are proportional and derivative gains, respectively, which shape the convergence of the error $\widetilde{\sigma}_k$. Equation \eqref{q ddot PD} is called a \textit{closed loop inverse kinematic} version of the equation \eqref{q ddot}. The controller having a similar structure is presented in \cite{siciliano1990kinematic}.\par
Note that in the above control input computation technique, the desired task trajectory $\sigma_{k,d}$ is often chosen manually. Due to the complexity of the architecture, it is often difficult to select an appropriate desired trajectory $\sigma_{k,d}$ and low-level controller gains $K_{p,k}$, $K_{d,k}$ to optimize the overall system performance. Moreover, the above method is only a local optimization technique as opposed to global optimization techniques which minimize some performance index across a whole trajectory and typically offer better solutions compared to local optimization approaches.

\section{Null-Space Projection}\label{Sec: task hierarchy}
Consider a robotic system with a control-affine dynamics:
 \begin{equation}\label{double integrator}
     \dot{x} = f(x) + G(x)u
 \end{equation}
where $x$ is the state of the system
, $u$ is the control input, $f(x)$ is a drift term, and $G(x)$ is the control coefficient. In the rest of the paper, for notational compactness, the functional dependencies on $x$, and $t$ are dropped whenever it is unambiguous. In the null-space projection technique, the control input $u_2$ corresponding to the second-priority task is projected onto the null space of the primary task using the formula:
\begin{equation}\label{a2_proj}
    u_2' = N_2^{}(q)u_2
\end{equation}
where $u_2'$ is the projected control input that does not interfere with the primary task. The null-space projector $N_2^{}(q)$ is obtained by evaluating $N_2^{}(q) = I - J_1^\dag(q) J_1(q)$,
where $J_1^\dag(q)$ is the right pseudo-inverse of the primary task’s Jacobian $J_1$, and $I$ is an identity matrix of suitable dimensions. Analogous to \eqref{a2_proj}, for the remaining tasks in the hierarchy ($2<k\leq K$), the control inputs are projected as $u_k' = N_k^{}(q)u_k$, with the null-space projectors recursively computed as:
\begin{align*}
    N_k^{}(q) & = N_{k-1}^{}(q)\left(I-J_{k-1}^\dag(q) J_{k-1}(q)\right)\quad 2\leq k\leq K
\end{align*}
and $N_1(q) = I$. Here $I$ is the identity matrix with suitable dimensions. Each task input is computed as if it were acting alone; then before adding its contribution to the overall system control input, a lower-priority task is projected onto the null space of the immediately higher-priority task so as to remove those control input components that would conflict with it. This technique guarantees that lower-priority objectives are constrained and therefore do not interfere with higher-priority objectives. As a result, the high-priority task is always achieved, and the lower ones are met only if they do not conflict with the task of higher priority. The final control input can be formulated by adding up the primary task control input and all the projected control inputs: 
\begin{align}\label{final acc}
    u = u_1 + \sum_{k=2}^K u_k' =  \sum_{k=1}^K N_k^{}(q)u_k. 
\end{align}
Plugging \eqref{final acc} into \eqref{double integrator}, we get
\begin{equation}\label{double integrator 2}
   \dot{x}(t) = f(x) + G(x)\sum_{k\in \mathcal{K}} N_k^{}(q)u_k.\end{equation}


A natural question arises of how many tasks can be handled simultaneously using this approach. Let us suppose that the primary task, of dimension $m_1$ is fulfilled by $n$ DOFs robotic system. The null space of its Jacobian (of full row rank) is a space of dimensions $n-m_1$. Supposing the secondary task of dimension $m_2$ does not conflict with the primary task (meaning that the secondary task acts in the null space of the primary task), the null space of their combination has dimension $n-m_1-m_2$. Choosing the tasks in a way they are not conflicting, it is useful to add tasks until $\sum m_k = n$. Thus, once all the degrees of freedom of the system are covered, it is useless to add successive tasks of lower priority since they will be projected onto an empty space (thus giving always a null contribution to the system control input). In case of conflicting tasks, it is not possible to make any generic assumption regarding the useful number of tasks but, case by case, the intersection among null spaces should be analyzed.

\section{Integration of Null Space Projection and Path Integral Control}

\begin{figure*}
    \centering
      \begin{tabular}{c c}
     \includegraphics[scale=0.31]{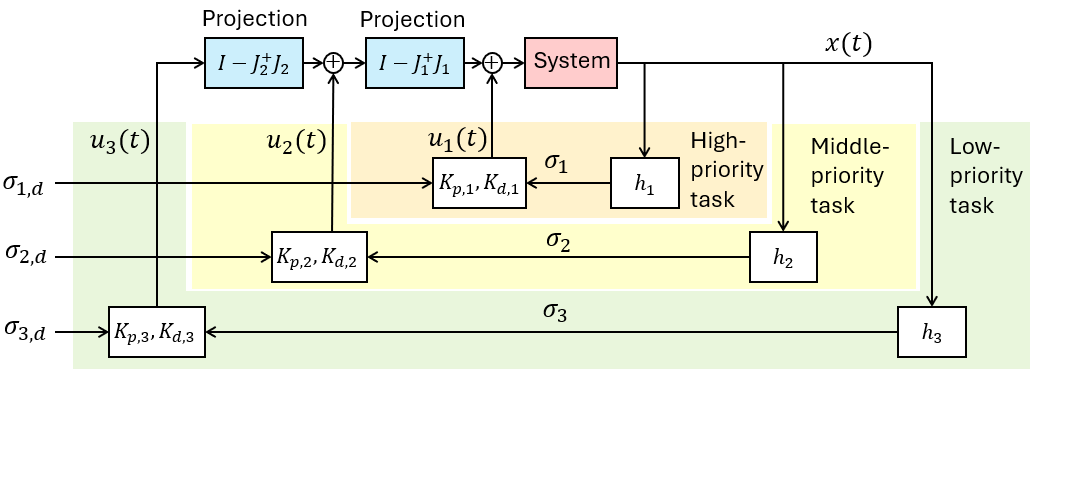} &\!\!\!\!\!\!\includegraphics[scale=0.31]{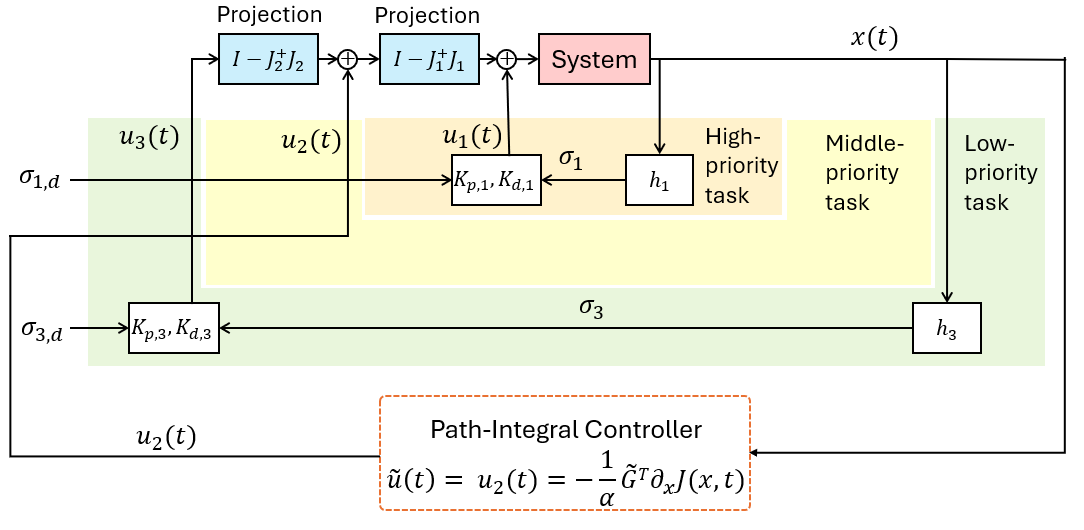} \\
     (a) Conventional task hierarchical control &(b) Integration of null-space projection and \\
     &  path integral control (proposed architecture)
      \end{tabular}
        \caption{Comparison of the conventional task hierarchical control approach and our proposed control approach, which integrates null-space projection with path integral control.} 
        \label{Fig. comparison}
\end{figure*}
In this section, we explain our proposed architecture. Consider a scenario where we need to accomplish $K$ tasks, however, only the $k$-th task is handled using a path integral controller. Here, we assume only one task is accomplished using the path integral controller. A similar formulation can be extended to cases where multiple tasks are controlled using path integral controllers. The control inputs for the other tasks $u_i, i\neq k$ are designed using PD controllers \eqref{q ddot PD}. Using \eqref{double integrator 2}, the system dynamics can be written as 
\begin{equation}\label{double integrator 3}
  \dot{x}(t) = f(x) + \!\!\!\sum_{i\in\mathcal{K}, i\neq k} \!\!\!\!G(x)N_iu_i(t) + G(x)N_ku_k(t).  
\end{equation}
Defining $\widetilde{G}(x) = G(x)N_k$, $\widetilde{u} = u_k$ and
\begin{align*}
    \widetilde{f}(x)& = f(x)+\sum_{i\in\mathcal{K}, i\neq k}\!\!\!\!G(x)N_iu_i(t),  
\end{align*}
we can rewrite \eqref{double integrator 3} as $d{x} = \widetilde{f}(x)dt + \widetilde{G}(x) \widetilde{u}\,dt.$ Note that the terms $\widetilde{f}(x)$ and $\widetilde{G}(x)$ are computed using the PD-like low-level controllers \eqref{q ddot PD} derived in Section \ref{Sec: preliminaries} and the null-space projection explained in Section \ref{Sec: task hierarchy}. Next, we derive a path integral controller for the control input $\widetilde{u}$. \par
First, we perturb $\widetilde{u}$ by adding a stochastic term: $\widetilde{u}\,dt \rightarrow \widetilde{u}\,dt + \hat{s}\,dw(t)$, where $w(t)$ is a standard Brownian motion and $\hat{s}\in \mathbb{R}$ is the diffusion coefficient. This perturbation allows the system to explore the policy space and discover a direction that minimizes the cost. The perturbed system dynamics are then given by 
\begin{equation}\label{SDE}
     d{x} = \widetilde{f}(x)dt + \widetilde{G}(x) \left(\widetilde{u}\,dt + \hat{s}\,dw(t)\right).
\end{equation}
Now we formulate the path integral control problem with the cost function: 
\begin{equation}\label{PI cost}
    \mathbb{E}_Q\left[\phi(x(T)) + \int_0^T  \left(L(x(t))+\frac{\alpha}{2}\|\widetilde{u}\|^2\right)dt\right]
\end{equation}
for some $\alpha>0$. The expectation $\mathbb{E}_Q$ is taken with respect to the dynamics \eqref{SDE}. The cost function has a quadratic control cost, an arbitrary state-dependent running cost $L(x(t))$, and a terminal cost $\phi(x(T))$. Consider the example of a platoon of robots illustrated in Section \ref{sec: Introduction}. Suppose task 2 (steering the platoon’s centroid toward a goal position) is controlled by the path integral controller. Let there be $I$ robots, with the configuration of each robot be denoted by $q_i, 1\leq i\leq I$, and the goal position be $q_g$. In this case, the running cost $L(x(t))$ could be $\|\sum_{i=1}^{I}\frac{1}{I}q_i(t)-q_g\|^2$, and the terminal cost $\phi(x(T))$ could be $\|\sum_{i=1}^{I}\frac{1}{I}q_i(T)-q_g\|^2$.\par
We now formally state our problem:
\begin{prob}\label{PI problem}
   \begin{align*}
       \min_{\widetilde{u}}\; & \mathbb{E}_{Q}\left[\phi(x(T))+\int_0^T \left(L(x(t))+\frac{\alpha}{2}\|\widetilde{u}\|^2\right)dt\right]\nonumber\\
       \emph{s.t.} \;\; &  d{x} = \widetilde{f}(x)dt + \widetilde{G}(x) \left(\widetilde{u}\,dt + \hat{s}\,dw(t)\right).
   \end{align*}
\end{prob}
It is well-known that the value function $J(x,t)$ of Problem \ref{PI problem} satisfies the Hamilton-Jacobi-Bellman (HJB) partial differential equation (PDE) \cite{williams2017model}: 
 \begin{equation}\label{HJB PDE}
  \begin{aligned}
         -\partial_tJ\!=&\!-\frac{1}{2\alpha}\!\left(\partial_xJ\right)^\top\!\!\widetilde{G}\widetilde{G}^\top\!\partial_xJ\!+\!L
         +\!\!\widetilde{f}^\top\!\partial_xJ\\
         &+\frac{\hat{s}^2}{2}\text{Tr}\left(\widetilde{G}\widetilde{G}^\top\partial^2_xJ\right),
          \end{aligned}
    \end{equation}
    with the boundary condition $J(x(T), T)=\phi(x(T))$. The optimal control is expressed in terms of the solution to this PDE as follows:
    \begin{equation}\label{optimal policy}
        \widetilde{u}(x,t) = -\frac{1}{\alpha}{\widetilde{G}}^\top\partial_xJ\left(x, t\right).
    \end{equation}
Therefore, in order to compute the optimal controls, we need to solve this backward-in-time PDE \eqref{HJB PDE}. Unfortunately, classical methods for solving partial differential equations (such as the finite difference method) of this nature suffer from the curse of dimensionality and are intractable for systems with more than a few state variables. The path integral control framework provides an alternative approach by transforming the HJB PDE into a path integral. This transformation allows us to approximate the solution using Monte Carlo simulations of the system's stochastic dynamics. We use the Feynman-Kac formula, which relates PDEs to path integrals \cite{williams2017model}. First, we define a constant $\lambda$ such that:
    \begin{equation}\label{lambda}
        \hat{s}^2 = \frac{\lambda}{\alpha}.
    \end{equation}
Next, using this constant $\lambda$, we introduce the following transformed value function $\xi(x,t)$:

\begin{equation}\label{exp transformation}
 J(x,t) = -\lambda\,\text{log}\left(\xi\left(x,t\right)\right).
\end{equation}
The transformation (\ref{exp transformation}) and \eqref{lambda} allow us to write the HJB PDE \eqref{HJB PDE} as a linear PDE in terms of $\xi(x,t)$ as
\begin{equation}\label{linear PDE}
    \partial_t\xi\!=\!\frac{L\xi}{\lambda}\!-\!\widetilde{f}^\top\partial_x\xi-\frac{\hat{s}^2}{2}\text{Tr}\left(\widetilde{G}\widetilde{G}^\top\partial^2_x\xi\right)
\end{equation}
with the boundary condition $\xi(x(T),T) =\exp\left(-\frac{\phi(x(T))}{\lambda}\right)$.
    This particular PDE is known as the \textit{backward Chapman–Kolmogorov PDE}. Now we find the solution of the linearized PDE \eqref{linear PDE} using the Feynman-Kac lemma.
\begin{lem}[Feynman-Kac lemma]\label{theorem: Feynman-Kac}
    The solution to the linear PDE \eqref{linear PDE} exists. Moreover, the solution is unique in the sense that $\xi$ solving \eqref{linear PDE} is given by \begin{equation}\label{xi}
       \begin{aligned}
  \!\!\!\!\!\!\xi\!\left(x,t\right) & \!= \!\mathbb{E}_{P}\!\! \left[\exp\!\left(\!\!-\frac{1}{\lambda}S(x,t)\!\!\right)\!\right]\\
  \end{aligned}
    \end{equation}
    where the expectation $\mathbb{E}_P$ is taken with respect to the uncontrolled dynamics of the system \eqref{SDE} (i.e., equation \eqref{SDE} with $\widetilde{u}=0$) starting at $x,t$. $S(x,t)$ is the cost to go of the state-dependent cost of a trajectory given by  
    \begin{equation*}
        S(x,t) = {\phi\left({{x}}({{T}})\right)}\!+\!\!\int_{t}^{{{T}}}\!\!\!\!L\!\left({{x}}(t)\right)\!dt .
    \end{equation*}
\end{lem} 
\begin{proof}
    The proof follows from \cite[Theorem 9.1.1]{oksendal2013stochastic}.
\end{proof}
We now obtain the expression for the optimal control policy for Problem \ref{PI problem} via Theorem \ref{thm: optimal_sol}. Before stating Theorem \ref{thm: optimal_sol}, we partition the system dynamics $   d{x} = \widetilde{f}(x)dt + \widetilde{G}(x) \left(\widetilde{u}\,dt + \hat{s}\,dw(t)\right)$ into subsystems that are directly and non-directly driven by the noise as:
\begin{equation}
\begin{bmatrix}
    dx^{(1)}\\dx^{(2)}
\end{bmatrix} = 
\begin{bmatrix}
    \widetilde{f}^{(1)}(x)\\\widetilde{f}^{(2)}(x)
\end{bmatrix}dt + \begin{bmatrix}
    \mathbf{0}\\\widetilde{G}^{(2)}(x)
\end{bmatrix}\left(\widetilde{u}\,dt + \hat{s}\,dw(t)\right).
\end{equation}

\begin{theorem}\label{thm: optimal_sol}
The optimal solution of Problem \ref{PI problem} exists, is unique and is given by
\begin{equation}\label{path integral control}
 \widetilde{u}^*(x,t)dt=\mathcal{G}\left(x\right)\frac{\mathbb{E}_{P}\left[\exp{\left(-\frac{1}{\lambda}S\right)}\hat{s}\,\widetilde{G}^{(2)}(x)\left(x\right)d{w}(t)\right]}{\mathbb{E}_{P}\left[\exp{\left(-\frac{1}{\lambda}S\right)}\right]}, 
\end{equation}
where the matrix $\mathcal{G}\left(x\right)$ is defined as $\mathcal{{G}}\!\left(x\right)\!=\widetilde{G}^{{(2)}^{\top}}\!(x)\!\left(\widetilde{G}^{(2)}(x)\widetilde{G}^{{(2)}^{\top}}\!(x)\right)^{-1}. $
\end{theorem}
\begin{proof}
   The existence and uniqueness of the optimal solution follow from the existence and uniqueness of the linear PDE \eqref{linear PDE} (Theorem \ref{theorem: Feynman-Kac}). The solution \eqref{path integral control} can be computed by taking the gradient of (\ref{xi}) with respect to $x$ \cite{theodorou2010generalized, williams2017model}. 
\end{proof}

To evaluate expectations in (\ref{xi}) and (\ref{path integral control}) numerically, we discretize the uncontrolled dynamics and use Monte Carlo sampling \cite{williams2017model}. After discretizing the uncontrolled dynamics, we get $ x_{t+1} = x_t + \widetilde{f}(x_t)\Delta t + \hat{s}\,\widetilde{G}(x_t)\epsilon\sqrt{\Delta t}$, where the term $\epsilon$ is a time-varying vector of standard normal Gaussian random variables and $\Delta t$ is the step size. The term $S(x,t)$ is approximately given as $ S(x,t) \approx {\phi\left({{x_T}}\right)}\!+\!\!\sum_{i=1}^{{{\tau}}}L\!\left({{x_t}}\right)\!\Delta t $, where $\tau=\frac{T-t}{\Delta t}$. Now suppose we generate $M$ samples of trajectories. Then we can approximate \eqref{path integral control} as 
\begin{equation}\label{u after discretization}
   \!\!\! \widetilde{u}_t^*(x_t)\approx\mathcal{G}\left(x_t\right)\frac{\sum_{j=1}^{M}\left[\exp{\left(-\frac{1}{\lambda}S(x,t)\right)}\hat{s}\,\widetilde{G}\left(x_t\right)\epsilon\right]}{\sum_{j=1}^{M}\left[\exp{\left(-\frac{1}{\lambda}S(x,t)\right)}\sqrt{\Delta t}\right]}.  
\end{equation}
Fig. \ref{Fig. comparison} shows the comparison of the conventional task hierarchical control approach and our proposed control approach, which integrates null-space projection with path integral control.
\begin{remark}
    Note that as the number of samples $M\rightarrow\infty$, \eqref{u after discretization} $\rightarrow$ \eqref{path integral control} i.e., path integral controller yields a globally optimal policy as $M\rightarrow\infty$.
\end{remark}
\begin{remark}
Increasing the diffusion coefficient $\hat{s}$ allows the system to explore a wider range of possibilities, but it also introduces greater noise into the control input. Therefore, $\hat{s}$ should be carefully selected to strike an appropriate balance between exploration and control noise.  
\end{remark}
The control input for each task in the hierarchy can be computed independently of the others. This allows for parallelization of the task hierarchy algorithm using GPUs, ensuring that the computational complexity remains manageable as additional tasks are introduced, provided a sufficient number of GPUs are available. As the number of agents increases, the dimensionality of the control inputs also grows. According to \cite{patil2024discrete}, in the case of the discrete-time path integral controller, the required sample size exhibits a \textit{logarithmic} dependence on the dimension of the control input. However, a formal analysis of sample complexity for the continuous-time path integral controller remains an open area for future research.

\section{Simulation Results}\label{Sec: simulations}
In this section, we present the simulation results for our proposed control framework.
\subsection{Single Agent Example}
Suppose a unicycle is navigating in a 2D space in the presence of an obstacle. The states of the unicycle model ${x}=[{p}_x \; {p}_y \; {s}\;\; \theta ]^\top$ consist of its $x-y$ position $p \coloneqq [{p}_x \; {p}_y]^\top$, speed ${s}$ and heading angle $\theta\in[0, 2\pi]$. The configuration of the system $q\coloneqq[{p}_x \; {p}_y \; \theta ]^\top$ consists of its position $p$ and the heading angle $\theta$. The system dynamics are given by the following equation:
\begin{equation}\label{unicycle}
    \begin{bmatrix}
        \dot{p}_x(t)\\ \dot{p}_y(t)\\ \dot{s}(t)\\ \dot{\theta}(t)
    \end{bmatrix} =  \begin{bmatrix} s(t)\cos (\theta(t)) \\ s(t)\sin (\theta(t))\\ 0\\ 0 \end{bmatrix} + \begin{bmatrix}
        0 & 0\\
        0 & 0\\
        1 & 0\\
        0 & 1
    \end{bmatrix} \begin{bmatrix} a(t)\\ \omega(t)  \end{bmatrix}.
\end{equation}
The control input $u\coloneqq[ a \; \omega ]^\top$ consists of acceleration $a$ and angular speed $\omega$. The simulation is set with $T=10$ seconds and $\Delta t = 0.01$ seconds. This unicycle system has to accomplish the following two tasks in descending order of importance: 1) obstacle avoidance and 2) move-to-goal. We design a proportional-derivative (PD) controller for the obstacle avoidance task and a path integral controller for the move-to-goal task.

\subsubsection{Obstacle-avoidance} 
Suppose the obstacle is placed at $c=[
    c_x \; c_y]^\top$ having radius $r$. 
In the presence of an obstacle in the advancing direction, the robot aims to keep it at a safe distance from the obstacle. The obstacle avoidance task becomes active when the robot is within a certain threshold distance and moving toward the obstacle. The threshold distance is greater than $r$ and is chosen arbitrarily. The task variable for obstacle avoidance $\sigma_1$ is defined as:
\begin{equation}\label{sigma1}
    \sigma_{1} =  \|p-c\|, 
\end{equation} 
and the desired value of the task variable be $\sigma_{1,d} =  r$. Differentiating \eqref{sigma1} with respect to $t$, we get 
\begin{equation}\label{sigma1_dot}
    \dot{\sigma}_1 = J_1v
\end{equation}
\begin{align*}
    J_{1} = \begin{bmatrix}
        \frac{p_{x}-c_x}{\|p-c\|} & \frac{p_{y}-c_y}{\|p-c\|} & 0
    \end{bmatrix},\quad v = \begin{bmatrix}
        s\cos\theta & s\sin\theta & \omega
    \end{bmatrix}^\top.
\end{align*}
In order to devise the control input for task 1 $u_1\coloneqq[ a_1 \; \omega_1 ]^\top$, we further differentiate \eqref{sigma1_dot} and get
\begin{align}\label{sigma1_ddot}
    \ddot{\sigma}_1 = \delta_1 + \Lambda_1\begin{bmatrix}
        a_1 & \omega_1
    \end{bmatrix}^\top,
\end{align}
 where $\delta_1$ and $\Lambda_1$ are defined as:
\begin{equation}\label{delta1}
\delta_1 = \frac{\left(p_y - c_y\right)^2 s^2 \cos^2\theta + \left(p_x - c_x\right)^2 s^2 \sin^2\theta} {\|p-c\|^3},    
\end{equation}
\begin{equation}\label{lambda1}
   \Lambda_1 = \begin{bmatrix}
        \frac{p_x-c_x}{\|p-c\|}\cos\theta + \frac{p_y-c_y}{\|p-c\|}\sin\theta \\ \frac{p_y-c_y}{\|p-c\|}s\cos\theta + \frac{p_x-c_x}{\|p-c\|}s\sin\theta
    \end{bmatrix}^\top.
\end{equation}
Adding the feedback term to \eqref{sigma1_ddot} similar to \eqref{q ddot PD}, we get
\begin{equation*}
    u_1 = \begin{bmatrix}
        a_1 \\ \omega_1
    \end{bmatrix} = \Lambda_1^\dagger\left(\ddot{\sigma}_{1,d} + K_{p,1} \widetilde{\sigma}_{1} + K_{d,1} \frac{d\widetilde{\sigma}_1}{dt} - \delta_1\right)
\end{equation*}
where $\widetilde{\sigma}_1 = \sigma_{1,d} - \sigma_{1} $, and $K_{p,1}$ and $K_{d,1}$ are the proportional and derivative gains respectively for task 1. 

\subsubsection{Move-to-goal} 
In this task, the robot must reach the goal position  $p_g \coloneqq [
    p_{g,x} \; p_{g, y}]^\top$. The task variable $\sigma_2$ is given as
\begin{equation}\label{sigma2}
    \sigma_2 = p = \begin{bmatrix}
        p_x \\p_y
    \end{bmatrix}
\end{equation}
and the desired value of the task variable will be $\sigma_{2,d} = p_g$. \par
First, we design the control input $u_2 = [
    a_2 \; \omega_2]^\top$ using a PD controller similar to task 1. Differentiating \eqref{sigma2} with respect to time, we get 
\begin{equation}\label{sigma2_dot}
    \dot{\sigma}_2 = J_2v,\quad  J_{2} = \begin{bmatrix}
        1 & 0 & 0\\
        0 & 1 & 0
    \end{bmatrix},\quad v = \begin{bmatrix}
        s\cos\theta \\ s\sin\theta\\ \omega
    \end{bmatrix}.
\end{equation}
Differentiating \eqref{sigma2_dot} further we get  
\begin{align}\label{sigma2_ddot}
    \ddot{\sigma}_2 = \Lambda_2\begin{bmatrix}
        a_2 \\ \omega_2
    \end{bmatrix}, \quad  \Lambda_2 = \begin{bmatrix}
        \cos\theta & -s\sin\theta\\
        \sin\theta & s\cos\theta
    \end{bmatrix}.
\end{align}
Adding the feedback term to \eqref{sigma2_ddot} similar to \eqref{q ddot PD}, we get
\begin{equation*}
    u_2 = \begin{bmatrix}
        a_2 \\ \omega_2
    \end{bmatrix} = \Lambda_2^\dagger\left(\ddot{\sigma}_{2,d} + K_{p,2} \widetilde{\sigma}_{2} + K_{d,2} \frac{d\widetilde{\sigma}_2}{dt}\right)
\end{equation*}
where $\widetilde{\sigma}_2 = \sigma_{2,d} - \sigma_{2} $, and $K_{p,2}$ and $K_{d,2}$ are the proportional and derivative gains respectively for task 2. 

\begin{figure}
    \centering
      \begin{tabular}{c c}
     \!\!\!\!\!\!\!\!\includegraphics[scale=0.3125]{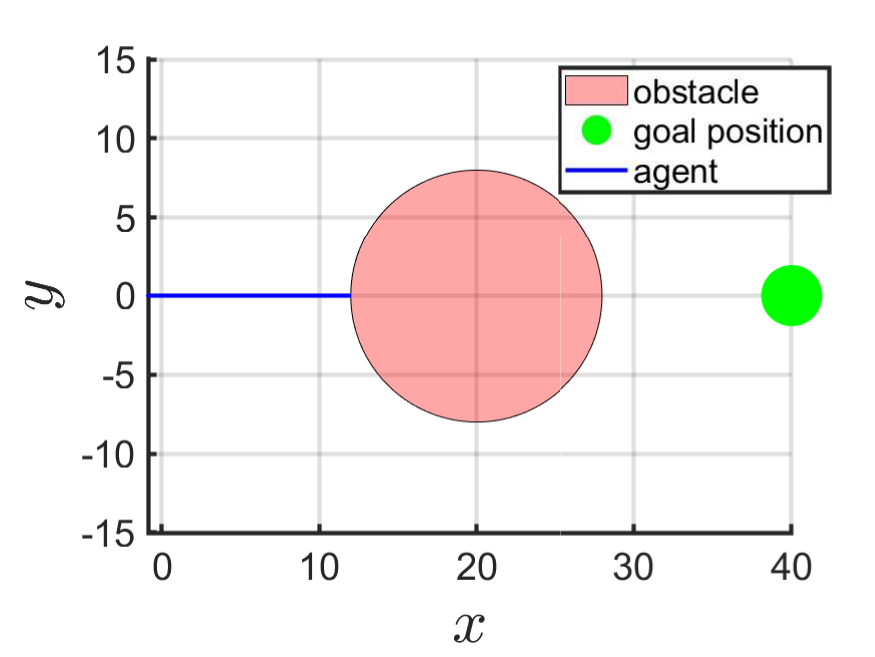} &\!\!\!\!\!\!\!\!\!\!\includegraphics[scale=0.3125]{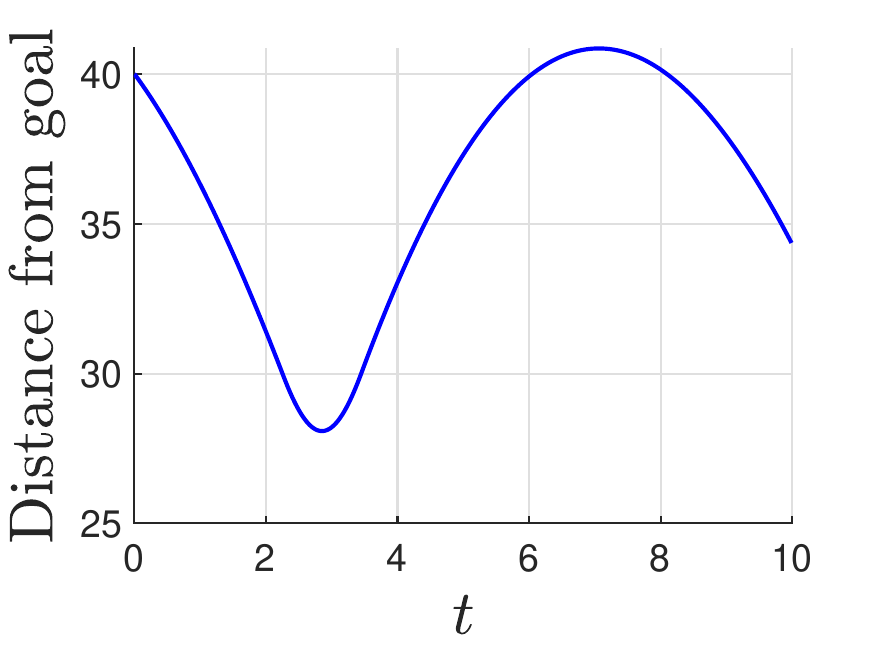} \\
     (a) Path followed without  & (b) Distance from the goal\\path integral controller & over time
      \end{tabular}
        \caption{Results of single-agent example without the path integral controller} 
        \label{Fig. single agent}
\end{figure}

Next, we design the control input $u_2$ via the path integral controller. The perturbed dynamics are given by \eqref{SDE} with the diffusion coefficient $\hat{s}=0.1$. We formulate the cost function \eqref{PI cost} with 
\begin{align*}
    L(x(t)) =  0.07\norm{p(t) - p_g },\;
    \phi(x(T)) = L(x(T)), \; \alpha =10.
\end{align*}
The number of Monte Carlo samples has been set to $10^4$. \par
Lastly, the control input $u_2$ is projected onto the null space of task 1 and we get the final control input as $ u = u_1 + (I - \Lambda_1^\dagger\Lambda_1)u_2$ where $I$ is the identity matrix of suitable dimensions. 
\begin{figure}
    \centering
      \begin{tabular}{c c}
     \!\!\!\!\!\!\!\!\!\!\includegraphics[scale=0.32]{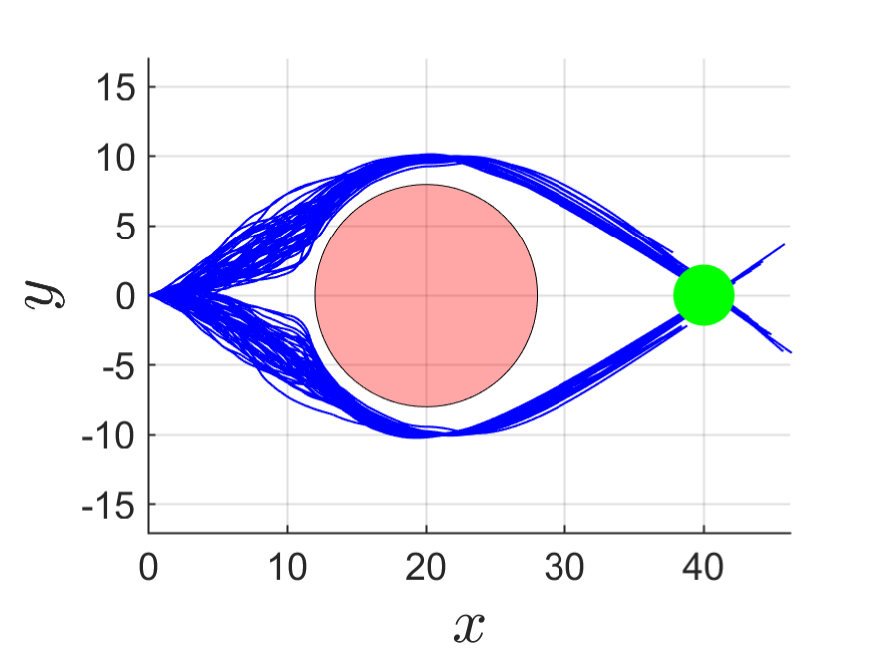} &\!\!\!\!\!\!\!\!\!\!\!\!\includegraphics[scale=0.32]{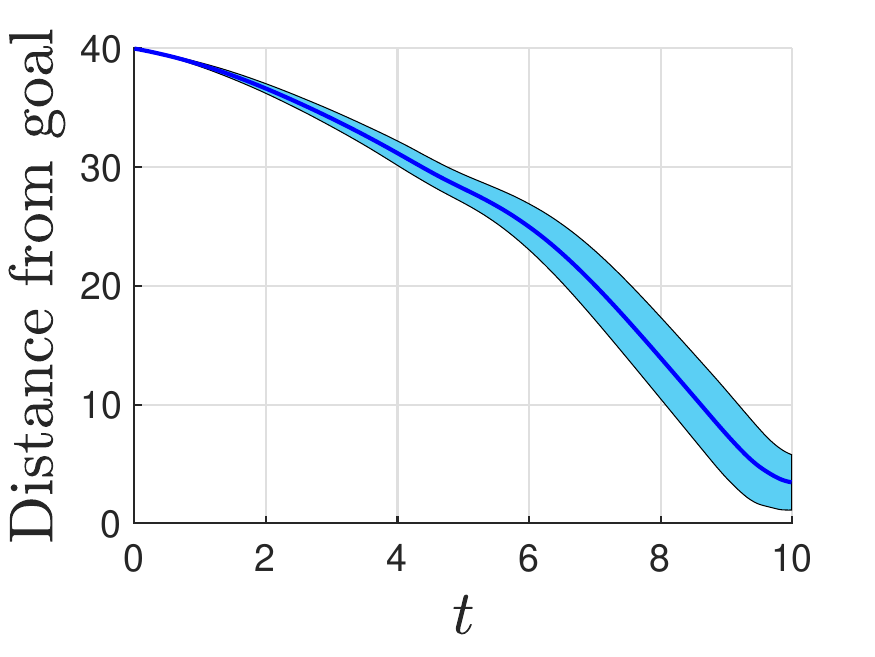} \\
     (a) Paths followed with  &\!\!\!\!\!\!\!\!\!\!\!\!\!\!\!\! (b) Mean distance from goal\\path integral controller &\!\!\!\!\!\!\!\!\!\!   $\pm$ standard deviation
      \end{tabular}
        \caption{Results of single-agent example using the path integral controller} 
        \label{Fig. single agent PI}
\end{figure}\par
Fig. \ref{Fig. single agent}(a) represents the results of the hierarchical controller when both tasks are controlled using PD controllers. As we can see, the robot initially moves towards the goal however, when it comes in the vicinity of the obstacle it drives itself away from the obstacle, as the obstacle avoidance task has a higher priority. The PD controllers get stuck in a local minima (the robot doesn't converge to the goal) and the robot keeps oscillating as shown in Fig. \ref{Fig. single agent}(b).\par
In Fig. \ref{Fig. single agent PI}, we show the results of the hierarchical controller when the first task is designed using a PD controller and the second task is designed using a path integral controller. Note that the path integral controller is designed by perturbing the control input $\widetilde{u}$ via a Brownian motion \eqref{SDE}. Therefore, the outcome of the path integral controller is stochastic. In Fig. \ref{Fig. single agent PI}(a), we plot 100 trajectories obtained by the devised hierarchical controller. We can see that all the trajectories successfully avoid the obstacle by going around it and reaching the goal without getting stuck in local minimas. Fig. \ref{Fig. single agent PI}(b) shows the mean distance to the goal over time, along with the standard deviation across the 100 trajectories. The solid blue line represents the mean over the 100 trajectories and the sky-blue shaded area represents one standard deviation from the mean. 

\subsection{Two Agents Example}
Suppose two unicycles are navigating in a 2D space in the presence of an obstacle. The unicycle dynamics follow the form in \eqref{unicycle} where the state for each unicycle is defined as ${x}^{(i)}=[{p}_{x}^{(i)} \; {p}_{y}^{(i)} \; {s}^{(i)}\;\; \theta^{(i)} ]^\top$, the configurations as $q^{(i)}=[{p}_{x}^{(i)} \; {p}_{y}^{(i)} \; \theta^{(i)} ]^\top$, and the control inputs as $u^{(i)}=[ a^{(i)} \; \omega^{(i)} ]^\top$ for $i=\{1,2\}$. The simulation is set with $T=10$ seconds and $\Delta t = 0.1$ seconds. This two-unicycle system has to accomplish the following three tasks in descending order of importance: 1) obstacle avoidance, 2) steering the centroid of the robots toward a goal position and 3) maintaining a specific distance between the two unicycles. We design a proportional-derivative (PD) controller for tasks 1) and 3), and a path integral controller for task 2).

\subsubsection{Obstacle-avoidance} 
Suppose the obstacle is placed at $c=[
    c_x \; c_y]^\top$ having radius $r$.
Similar to the single agent case, the obstacle avoidance task becomes active when any of the robots is within a certain threshold distance and moving toward the obstacle. Let $\sigma_{1}^{(i)}$ represent the task variable for the robot $i$ and $\sigma_{1,d}$ represent the desired task value for both the robots. $\sigma_{1}^{(i)}$ and $\sigma_{1,d}$ can be written as 
\begin{align}\label{sigma1 2agent}
    \sigma_{1}^{(i)} =  \|p^{(i)}-c\|, \quad i=\{1,2\},\quad     \sigma_{1,d} =  r.
\end{align} 
Similar to the single agent case, we differentiate \eqref{sigma1 2agent} twice with respect to time for both robots and get 
\begin{align}\label{sigma1_ddot 2agent}
    \ddot{\sigma}_{1}^{(i)} = \delta_{1}^{(i)} + \Lambda_{1}^{(i)}\begin{bmatrix}
        a_1^{(i)} & \omega_1^{(i)}
    \end{bmatrix}^\top
\end{align}
where $\delta_1^{(i)}$ and $\Lambda_1^{(i)}$ can be defined similar to \eqref{delta1}, \eqref{lambda1}:



Finally, adding the feedback term to \eqref{sigma1_ddot 2agent} similar to \eqref{q ddot PD}, we get the control input for obstacle avoidance:
\begin{equation*}
    u_1^{(i)} = \begin{bmatrix}
        a_1^{(i)} \\ \omega_1^{(i)}
    \end{bmatrix} = \Lambda_1^{(i)^\dagger}\!\!\!\left(\!\!\ddot{\sigma}_{1,d} + K_{p,1}^{(i)} \widetilde{\sigma}_{1}^{(i)} + K_{d,1}^{(i)} \frac{d\widetilde{\sigma}_1^{(i)}}{dt} - \delta_1^{(i)}\!\!\right)
\end{equation*}
where $\widetilde{\sigma}_1^{(i)} = \sigma_{1,d} - \sigma_{1}^{(i)} $, and $K_{p,1}^{(i)}$ and $K_{d,1}^{(i)}$ are the proportional and derivative gains respectively for task 1 and for $i=\{1,2\}$. Also, note that 
\begin{equation*}
    \Lambda_1 = \begin{bmatrix}
        \Lambda_1^{(1)} & 0\\
        0 &  \Lambda_1^{(2)}
    \end{bmatrix}, \quad \Lambda_1^\dagger = \begin{bmatrix}
        \Lambda_1^{(1)^\dagger} & 0\\
        0 &  \Lambda_1^{(2)^\dagger}
    \end{bmatrix}.
\end{equation*}

\subsubsection{Steering the robots' centroid toward a goal position}
This task involves steering the centroid of the robots toward a goal position $p_g$. The task variable $\sigma_2$ is given as
\begin{equation}\label{sigma2 2agents}
    \sigma_2 = \frac{p^{(1)} + p^{(2)}}{2}
\end{equation}
and the desired value of the task variable will be $\sigma_{2,d} = p_g$. \par
First, we design the control input $u_2 = [
    a_2^{(1)} \; \omega_2^{(1)} \; a_2^{(2)} \; \omega_2^{(2)}]^\top$ using a PD controller. Differentiating \eqref{sigma2 2agents} twice with respect to time we get 
\begin{align}\label{sigma2_ddot 2agents}
    \!\!\ddot{\sigma}_2 = \Lambda_2\!\!\begin{bmatrix}
        a_2^{(1)} \\ \omega_2^{(1)} \\ a_2^{(2)} \\ \omega_2^{(2)}
    \end{bmatrix}\!\!, \quad \!\!\!\!\! \Lambda_2 \!\!=\!\!\frac{1}{2} \!\!\begin{bmatrix}
        \cos\theta^{(1)} & \sin\theta^{(1)}\\
        -s^{(1)}\sin\theta^{(1)} & s^{(1)}\cos\theta^{(1)}\\
        \cos\theta^{(2)} & \sin\theta^{(2)}\\
        -s^{(2)}\sin\theta^{(2)} & s^{(2)}\cos\theta^{(2)}\\
    \end{bmatrix}^\top\!\!\!\!\!.
\end{align}
Adding the feedback term to \eqref{sigma2_ddot 2agents} similar to \eqref{q ddot PD}, we get
\begin{equation*}
    u_2 = \begin{bmatrix}
        a_2^{(1)} \\ \omega_2^{(1)}\\a_2^{(2)} \\ \omega_2^{(2)}
    \end{bmatrix} = \Lambda_2^\dagger\left(\ddot{\sigma}_{2,d} + K_{p,2} \widetilde{\sigma}_{2} + K_{d,2} \frac{d\widetilde{\sigma}_2}{dt}\right)
\end{equation*}
where $\widetilde{\sigma}_2 = \sigma_{2,d} - \sigma_{2} $, and $K_{p,2}$ and $K_{d,2}$ are the proportional and derivative gains respectively for task 2. \par

Next, we design the control input for task 2) via a path integral controller. The perturbed dynamics are given by \eqref{SDE} with the diffusion coefficient $\hat{s}=0.1$. We formulate the cost function \eqref{PI cost} with $\alpha =10$,
\begin{align*}
    \!\!L(x(t))\! = \! 0.21\!\!\norm{\frac{p^{(1)}(t) \!+ \!p^{(2)}(t)}{2} \!\!-\! p_g }\!,\;
    \phi(x(T)) = L(x(T)).
\end{align*}
The number of Monte Carlo samples has been set to $10^4$. 
\subsubsection{Maintaining a specific distance between two unicycles}
In this task, the two unicycles must maintain the distance $l$ between each other. The task variable $\sigma_3$ is given as
\begin{equation}\label{sigma3 2agents}
    \sigma_3 = \frac{1}{2}\left(p^{(1)} - p^{(2)}\right)^\top\left(p^{(1)} - p^{(2)}\right) 
\end{equation} 
and the desired value of the task variable will be $\sigma_{3,d} = \frac{l^2}{2}$.
Differentiating \eqref{sigma3 2agents} twice with respect to time we get 

\begin{equation}\label{sigma3_ddot 2agents}
  \ddot{\sigma}_{3} = \delta_{3} + \Lambda_{3}\begin{bmatrix}
        a_3^{(1)} & \omega_3^{(1)} & a_3^{(2)} & \omega_3^{(2)}
    \end{bmatrix}^\top
\end{equation}
\[\delta_3 \!= \!s^{(1)^2} \!\!\!-2s^{(1)}s^{(2)}\!\!\left(\!\cos\theta^{(1)}\!\!\cos\theta^{(2)}\!\!+\!\sin\theta^{(1)}\!\!\sin\theta^{(2)}\!\right)\!\!+\!\! s^{(2)^2},\] 
\begin{equation*}
  \!\!\! \Lambda_3 \!\! =\!\! \begin{bmatrix}
        \left(p_x^{(1)}\!\!-\!p_x^{(2)}\right)\cos\theta^{(1)}\!\! + \!\! \left(p_y^{(1)}\!\!-\!p_y^{(2)}\right)\sin\theta^{(1)}\\
        
        \left(p_x^{(2)}\!\!-\!p_x^{(1)}\right)s^{(1)}\sin\theta^{(1)} \!\!+ \!\!\left(p_y^{(1)}\!\!-\!p_y^{(2)}\right)s^{(1)}\cos\theta^{(1)}\\

        \left(p_x^{(2)}\!\!-\!p_x^{(1)}\right)\cos\theta^{(2)}\!\! + \!\!\left(p_y^{(2)}\!\!-\!p_y^{(1)}\right)\sin\theta^{(2)}\\

        \left(p_x^{(1)}\!\!-\!p_x^{(2)}\right)s^{(2)}\sin\theta^{(2)}\!\! +\!\! \left(p_y^{(2)}\!\!-\!p_y^{(1)}\right)s^{(2)}\cos\theta^{(2)}
    \end{bmatrix}^{\!\!\top}\!\!\!\!\!. 
\end{equation*}
Finally, adding the feedback term to \eqref{sigma3_ddot 2agents} similar to \eqref{q ddot PD}, we get
\begin{equation*}
    u_3 = \begin{bmatrix}
        a_3^{(1)} \\ \omega_3^{(1)}\\a_3^{(2)} \\ \omega_3^{(2)}
    \end{bmatrix} = \Lambda_3^\dagger\left(\ddot{\sigma}_{3,d} + K_{p,3} \widetilde{\sigma}_{3} + K_{d,3} \frac{d\widetilde{\sigma}_3}{dt}-\delta_3\right)
\end{equation*}
where $\widetilde{\sigma}_3 = \sigma_{3,d} - \sigma_{3} $, and $K_{p,3}$ and $K_{d,3}$ are the proportional and derivative gains respectively for task 3.\par

Lastly, the control input of the second task $u_2$ is projected onto the null space of task 1, and the control input of task 3 $u_3$ is projected onto the null spaces of both task 1 and task 2. This yields the overall control input as:
\begin{equation*}
    u = u_1 + (I - \Lambda_1^\dagger\Lambda_1)\left(u_2 + (I - \Lambda_2^\dagger\Lambda_2)u_3 \right)
\end{equation*}
where $I$ is the identity matrix of suitable dimensions. \par
In this experiment, we set the desired distance between the agents to $l = 0.5$ with the initial positions of the unicycles at $[
    p_x^{(1)} \; p_y^{(1)}\;  p_x^{(2)} \;  p_y^{(2)}]^\top = [
    -4.5 \;\; 0 \; -4 \;\; 0]^\top$. Fig. \ref{Fig. two agent}(a) represents the results of the hierarchical controller when all three tasks are controlled using PD controllers. As we can see, the centroid of the two unicycles initially moves towards the goal however, when the unicycles come in the vicinity of the obstacle they drive themselves away from the obstacle, as the obstacle avoidance task has a higher priority. Due to the limitations of the PD controllers, which get stuck in local minima, the unicycles begin to oscillate and never successfully cross the obstacle to reach the goal, as seen in Fig. \ref{Fig. two agent}(b).

\begin{figure}
    \centering
      \begin{tabular}{c c}
     \!\!\!\!\!\!\!\!\!\!\includegraphics[scale=0.32]{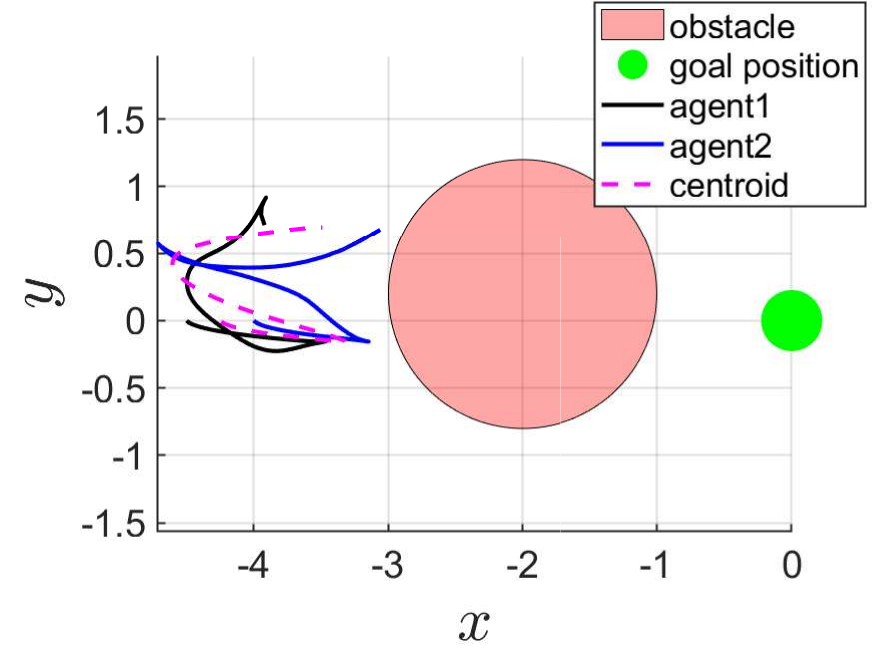} &\!\!\!\!\!\!\!\includegraphics[scale=0.32]{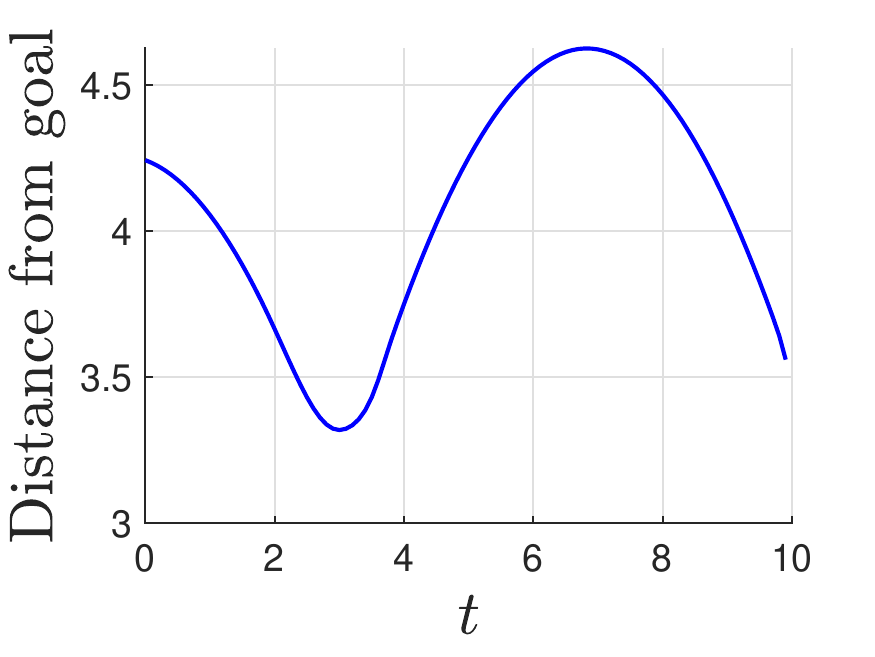} \\
     (a) Path followed without  & (b) Distance from the goal\\path integral controller & over time
      \end{tabular}
        \caption{Results of two-agents example without the path integral controller} 
        \label{Fig. two agent}
\end{figure}

In Fig. \ref{Fig. two agent PI}, we present the results using the hierarchical controller, where tasks 1) and 3) are managed with PD controllers, while task 2) is handled using a path integral controller. In Fig. \ref{Fig. two agent PI}(a), we plot 100 trajectories of the two agents and their centroid. The hierarchical controller successfully ensures obstacle avoidance and steers the centroid towards the goal in all trajectories, without getting stuck in local minima. Fig. \ref{Fig. two agent PI}(b) shows the mean distance of the centroid from the goal over time, along with the standard deviation across the 100 trajectories. The solid blue line represents the mean and the sky-blue shaded area represents one standard deviation from the mean. 

\begin{figure}
    \centering
      \begin{tabular}{c c}
     \!\!\!\!\!\!\!\!\!\!\includegraphics[scale=0.32]{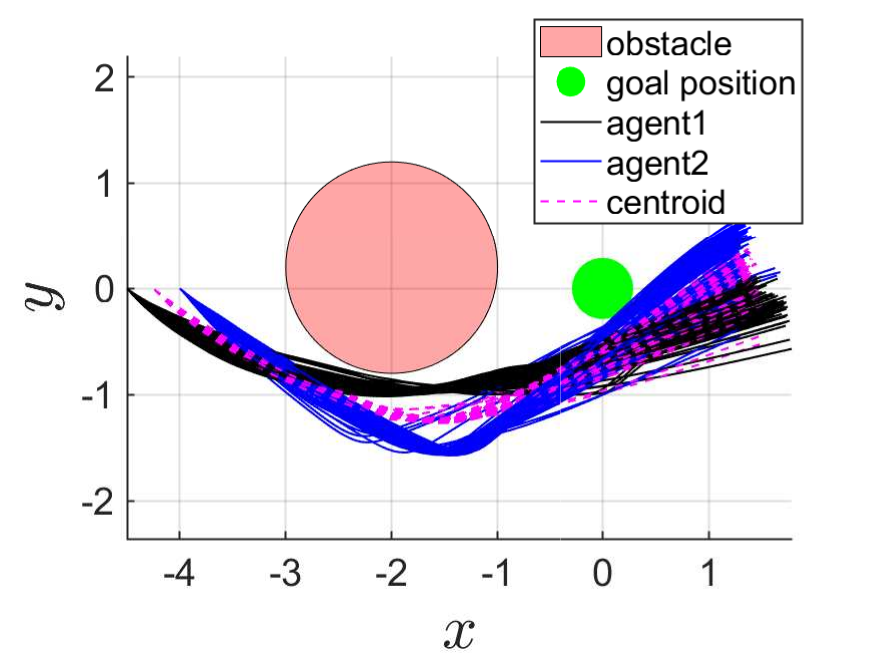} &\!\!\!\!\!\!\!\!\!\!\includegraphics[scale=0.32]{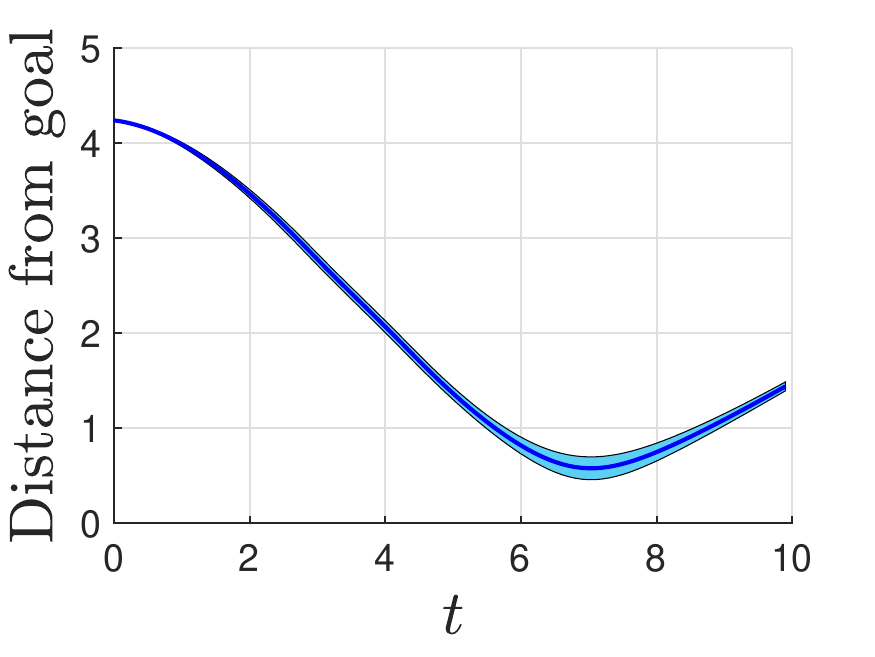} \\
     (a) Paths followed with  & \!\!\!\!\!\!\!\!\!\! (b) Mean distance from goal\\path integral controller & $\pm$ standard deviation
      \end{tabular}
        \caption{Results of two-agents example using the path integral controller} 
        \label{Fig. two agent PI}
\end{figure}\par

\begin{figure}
    \centering
      \begin{tabular}{c c}
     \!\!\!\!\!\!\!\!\!\!\includegraphics[scale=0.31]{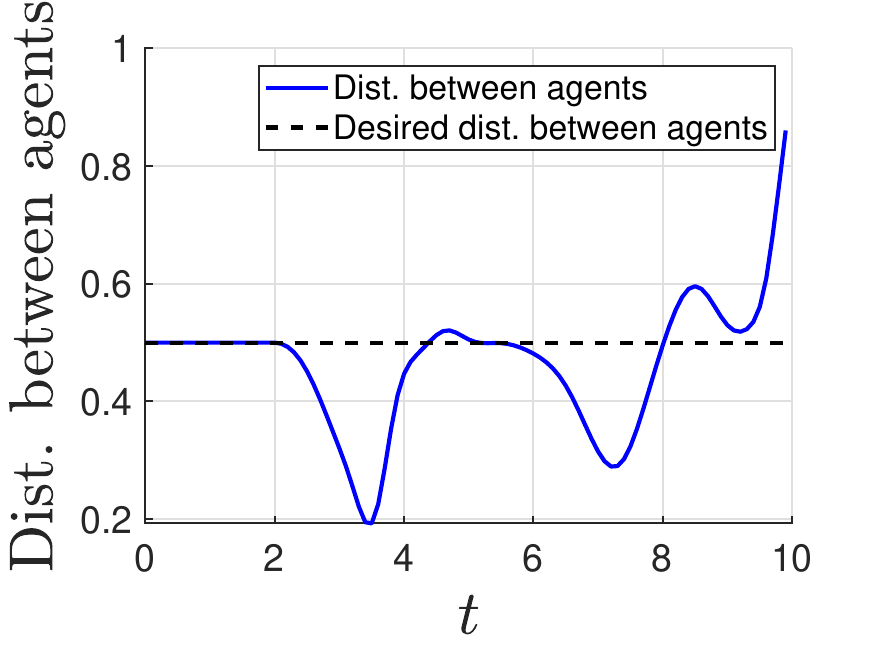} &\!\!\!\!\!\!\!\!\!\!\includegraphics[scale=0.31]{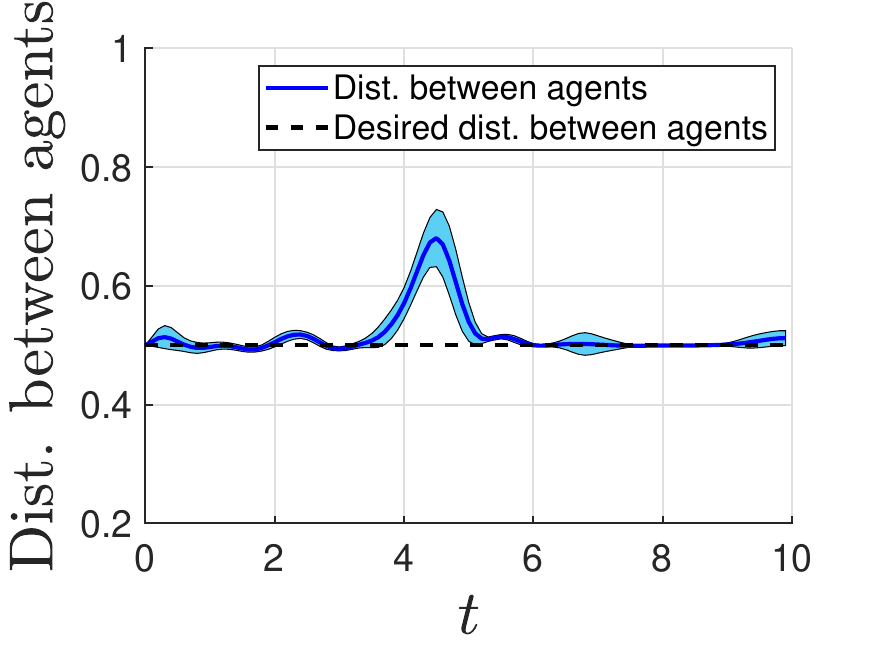} \\
     \!\!\!\!\!\!\!\!\!(a) Without path integral&\!\!\!\!\!\!   (b) With path integral \\controller: Distance between & controller: Mean distance  \\agents & between agents  \\  & $\pm$ standard deviation
      \end{tabular}
        \caption{Distance between agents over time} 
        \label{Fig. dist bet agents}
\end{figure}
In Fig. \ref{Fig. dist bet agents}(a), we show the distance between the agents over time when the controller for task 2) is designed using a PD controller. Fig. \ref{Fig. dist bet agents}(b) shows the mean distance between the agents over time, using the path integral controller for task 2). The solid blue line represents the mean across 100 trajectories, while the shaded area corresponds to one standard deviation. This figure highlights that the path integral controller helps the agents better maintain the desired distance between each other compared to the PD controller.

\section{Conclusion}
In this paper, we presented a novel control framework that combines the null-space projection technique with the path integral control method to solve hierarchical task control problems for robotic systems. The null-space projection allows for the prioritization of multiple tasks, ensuring that higher-priority tasks are always satisfied. We devised path integral controllers for more complex tasks in the hierarchy, and the remaining simpler tasks were accomplished using PID-like local controllers. The path integral controller effectively handles nonlinearity in the dynamics and the cost functions. The proposed framework overcomes the local minimas, which is a limitation of traditional local controllers, and successfully achieves multiple tasks in a hierarchical structure. \par 

For future work, we aim to incorporate the importance sampling techniques into the path integral controller, which offers improved sample efficiency compared to the standard approach. Another idea is use the Hamiltonian
dynamics to improve the sampling efficiency \cite{akshay2020hamiltonian}. Additionally, we plan to extend the application of this framework to large-scale multi-robot systems and robotic manipulators with a high number of degrees of freedom.






\bibliographystyle{IEEEtran}
\bibliography{bibliography}

\end{document}